\documentclass{article}
\usepackage{jmlr2en}
\usepackage{amsmath}

\newcommand{\bE}{{\mathbf{E}}}
\newcommand{\bI}{{\mathbf{I}}}
\newcommand{\bJ}{{\mathbf{J}}}
\newcommand{\bM}{{\mathbf{M}}}
\newcommand{\bP}{{\mathbf{P}}}
\newcommand{\bo}{{\mathbf{1}}}

\newcommand{\enumr}{\begin{enumerate}[label=\roman{*})]}
\newcommand{\enumR}{\begin{enumerate}[label=\Roman{*})]}
\newcommand{\enuma}{\begin{enumerate}[label=\alph{*})]}

\newcommand{\cc}{\cite}

\newcommand{\argmax}{\mbox{\,\rm arg\,max}}

\newcommand{\ttext}[1]{\mathrm{#1}}

\newcommand{\cpi}{\ensuremath {\pi_{\text{CHK}}}}
\newcommand{\upi}{\ensuremath {\pi_{\text{CK}}}}
\newcommand{\krpi}{\ensuremath {\pi_{\text{KR}}}}
\newcommand{\api}{\ensuremath {\pi_{\ttext{BK}}}}
\newcommand{\cui}{\ensuremath {u^i_{\ttext{CK}}}}

\newcommand{\aui}{\ensuremath {u^i_{\ttext{BK}}}}
\newcommand{\ui}{\ensuremath {u^i}}
\newcommand{\supp}{\ttext{Sp}}

\firstpageno{1}

\jmlrheading{1}{2015}{1-48}{4/00}{10/00}{Wesley  Cowan and Michael N. Katehakis }

\ShortHeadings{An Optimal ISM Policy for  Uniform Bandits of Unknown Support}{Cowan and Katehakis}
\firstpageno{1}
\begin{document}

\title{An Asymptotically  Optimal    Policy for  Uniform Bandits of \\ Unknown Support}
\author{\name Wesley Cowan \email  cwcowan@math.rutgers.edu \\
       \addr   Department of Mathematics\\ Rutgers University\\
110 Frelinghuysen Rd., Piscataway, NJ 08854, USA 
       \AND
       \name Michael N. Katehakis  \email mnk@rutgers.edu  \\
       \addr Department of Management Science and Information Systems\\
       Rutgers University\\
 100 Rockafeller Rd., Piscataway, NJ 08854,  USA}


\maketitle

\begin{abstract}
Consider the problem of a controller sampling sequentially from a finite number of $N \geq 2$ populations,
 specified by
random variables  $X^i_k$, $ i = 1,\ldots , N,$ and $k = 1, 2, \ldots$;   where $X^i_k$ denotes
the outcome from population $i$ the $k^{th}$ time it is sampled. It is assumed that for each fixed $i$, 
 $\{ X^i_k \}_{k \geq 1}$ is  a sequence of i.i.d. uniform random variables over some interval $[a_i, b_i]$, with the support (i.e., $a_i, b_i$) unknown to the controller.  The  objective is to have a policy $\pi$  for deciding, based on available data, from which of  the $N$ populations    
 to sample from at any time $n=1,2,\ldots$ so as     
  to maximize the expected sum of outcomes of  $n$ samples or 
 equivalently to minimize the regret due to lack on information of the parameters  $\{ a_i \}$ and   $\{ b_i \}$.  In this paper, we present a simple UCB-type policy that is asymptotically optimal. Additionally, finite horizon regret bounds are given. 
%
\end {abstract}
 

\chapter{Introduction and Summary}
\section{Main Model} 
Let $\mathcal{F}$ be a known family of probability densities on $\boldmath{R}$, each with finite mean. We define $\mu(f)$ to be the expected value under density $f$, and $\supp(f)$ to be the support of $f$. Consider the problem of sequentially sampling from a finite number of $N \geq 2$ populations or `bandits', where measurements from population $i$ are specified by an i.i.d. sequence of random variables $\{ X^i_k \}_{ k \geq 1 }$ with density $f_i \in \mathcal{F}$.  We take each $f_i$ as unknown to the controller. It is convenient to define, for each $i$, $\mu_i = \mu(f_i)=\int_{ \supp(f) } \, x f(x) dx ,$ \, and 
$\mu^* =  \mu^* (\{ f_i \}) =\, \max_i \mu(f_i)$. Additionally, we take $\Delta_i = \mu^* - \mu_i \geq 0$, the discrepancy of bandit $i$.

We note, but  for simplicity will not consider explicitly, that 
both discrete and continuous distributions can be studied when one takes $\{ X^i_k \}_{ k \geq 1 }$ to be   i.i.d.  with density $f_i$, with  
 respect to some known measure $\nu_i.$ 

For any adaptive, non-anticipatory policy $\pi$, $\pi(t) = i$ indicates that the controller samples bandit $i$ at time $t$. Define $T^i_\pi(n) = \sum_{t = 1}^n \boldmath{1}\{ \pi(t) = i \}$, denoting the number of times bandit $i$ has been sampled during the periods $t = 1, \ldots, n$ under policy $\pi$;  we take, as a convenience, $T^i_\pi(0) = 0$ for all $i, \pi$.
The  {\sl value} of a policy $\pi$ is the expected sum of the first $n$ outcomes
under $\pi$, which we define to be the  function $V_\pi(n) :$
\begin{equation} \label{eqn:vn}
V_\pi(n)  = 
\bE\left[  \sum_{i=1}^N     \sum_{k = 1}^{T^i_\pi(n)}  X^i_k  \right] 
= \sum_{i=1}^N      \mu_i {\bE}\left[ T^i_\pi(n) \right],
\end{equation}
where for simplicity the dependence of  $V_\pi(n)$ on  the unknown densities $\{ f_i \}$ is suppressed.  
  The {\sl regret} of a policy is taken to be the expected loss due to ignorance of the
  underlying distributions by the controller. Had the controller complete information, she would at every round activate some bandit $i^*$ such that $\mu_{i^*} = \mu^*= \max_i \mu_i$. For a given policy $\pi$, we   define the expected regret of that policy at time $n$ as   

\begin{equation} \label{eqn:regret}
R_\pi(n) = n \mu^* -V_\pi(n)
= \sum_{i = 1}^n \Delta_i \bE\left[ T^i_\pi(n) \right].
\end{equation}
 
 We are interested in policies for which $V_\pi(n)$ grows as fast  as possible with $n$, or equivalently that $R_\pi(n)$ grows as slowly as possible with $n .$

\section{Preliminaries - Background} 
We restrict $\mathcal{F}$ in the following way:

{\bf Assumption 1.} Given any  set of bandit densities $\{ f_i \}_{i=1}^N$, for any sub-optimal bandit $i ,$ i.e., $\mu_i(f_i) \neq \mu^*(\{ f_i \}) ,$ there exists some $\tilde{f}_i \in \mathcal{F}$ such that $\supp(\tilde{f}_i) \supset \supp(f_i)$, and $\mu( \tilde{f}_i ) > \mu^*(\{ f_i \}) $. 

Effectively, this ensures that at any finite time, given a set of bandits under consideration, for any bandit there is a density in $\mathcal{F}$ that would both potentially explain the measurements from that bandit, and make it the unique optimal bandit of the set.

The focus of this paper  is  on $\mathcal{F}$ as the set of uniform densities over some unknown support.

Let   $\bI(f, g)$ denote the Kullback-Liebler divergence of density $f$ from $g$,
\begin{equation}
\bI(f, g) = \int_{ \supp(f) } \ln \left( \frac{f(x)}{g(x)} \right) f(x) dx = \bE_f \left[ \ln \left( \frac{f(X)}{g(X)} \right) \right].
\end{equation}  It is a simple generalization of a classical result 
(part 1 of Theorem 1) of \cite{bkmab96} that if a policy 
   $\pi$ is uniformly fast (UF), i.e., $R_\pi(n) = o(n^\alpha)$ for all $\alpha > 0$ and for any choice of $\{ f_i \} \subset \mathcal{F}$,  then,  the following bound holds:

\begin{equation}\label{eqn:lower-bound}
\liminf_n \frac{ R_\pi(n) }{ \ln n } \geq \bM_{\ttext{BK}}( \{ f_i \} )  , 
\mbox{ for all $\{ f_i \} \subset \mathcal{F}$,}
\end{equation}
where the bound $ \bM_{\ttext{BK}}( \{ f_i \} ) $ itself    is determined by the specific distributions of the populations:
\begin{equation}\label{eqn:general-boundm}
  \bM_{\ttext{BK}}( \{ f_i \} ) =  
 \sum_{i:\mu_i \neq \mu^*} \frac{ \Delta_i }{ \inf_{g \in \mathcal{F}} \{ \bI(f_i, g) : \mu(g) \geq \mu^* \} } .
\end{equation}

For a given set of densities $\mathcal{F}$, it is of   interest to construct policies $\pi$ such that $$\lim_n \frac{R_\pi(n)}{ \ln n} =\bM_{\ttext{BK}}( \{ f_i \} )  , 
\mbox{ for all $\{ f_i \} \subset \mathcal{F}$}   .$$ 
Such policies achieve the slowest (maximum) regret (value) growth rate possible among UF policies. They have been called UM or asymptotically optimal or efficient, cf. \cc{bkmab96}.

For a given $f \in \mathcal{F}$, let $\hat{f}_k \in \mathcal{F}$ be an estimator of $f$ based on the first $k$ samples from $f$. It was shown in  \cite{bkmab96} that under sufficient conditions on $\{ \hat{f}^i_k \}$, asymptotically optimal (UM) UCB-policies could be constructed by initially sampling each bandit some number of $n_0$ times, and then  for $n \geq n_0$, following an index policy:
\begin{equation}
\pi^0(n+1) = \argmax_i \, \{\ui(n, T^i_{\pi^0}(n) )\} ,
\end{equation}
where the indices $\ui(n,t)$ are 
`inflations of the current estimates for the means' (ISM), specified  as:
\begin{equation}\label{eqn:bk-index}
\ui(n,t)=\aui(n, t,\hat{f}^i_t)=\sup_{g\in \mathcal{F}}\left\{\mu(g ) \ : \   \bI(\hat{f}^i_t, g) < \frac{\ln n  }{t}\right\} .
\end{equation}

The sufficient conditions on the estimators $\{ \hat{f}^i_k \}$ are as follows:

Defining
$$\bJ(f, c) = \inf_{g\in \mathcal{F}} \{ \bI(f, g) : \mu(g) > c \},$$
for all choices of $\{ f_i \} \subset \mathcal{F}$ and all $\epsilon > 0$, $\delta > 0$, the following hold for each $i,$  as $k \to \infty .$ 
\begin{itemize}
\item[ C1:] $\boldmath{P}\left( \bJ(\hat{f}^i_k, \mu^* - \epsilon) < \bJ(f_i, \mu^* - \epsilon ) - \delta\right) = o(1/k) .$
\item[ C2:] $\boldmath{P}\left( \aui(k, j, \hat{f}^i_j) \leq \mu_i - \epsilon, \ttext{ for \ some \ } j \in \{n_0, \ldots,   k\} \right) = o(1/k) .$
\end{itemize}

These conditions correspond to Conditions A1-A3 given in \cite{bkmab96}. However under the stated Assumption 1 on $\mathcal{F}$ given here, Condition A1 therein is automatically 
 satisfied. Conditions A2 (see also Remark 4(b) in \cite{bkmab96}) and A3    are given as C1 and C2,    above, respectively.
Note, Condition (C1) is essentially satisfied as long as $\hat{f}^i_k$ converges to $f_i$ (and hence $\bJ(\hat{f}^i_k, \mu^*-\epsilon) \to \bJ(f_i, \mu^*-\epsilon ))$ sufficiently quickly with $k$. This can often be verified easily with standard large deviation principles. The difficulty in proving the optimality of policy $\pi^0$ is often in verifying that Condition (C2) holds.

\begin{remark}
The above discussion is a parameter-free variation of that in \cc{bkmab96}, where $\mathcal{F}$ was taken to be parametrizable, i.e., $\mathcal{F} = \{ f_{\underline{\theta}} : \underline{\theta} \in \Theta \}$, taking $\underline{\theta}$ as a vector of parameters in some parameter space $\Theta$.  Further, \cc{bkmab96} considered potentially different parameter spaces (and therefore potentially different parametric forms) for each bandit $i$. There, Conditions A1-A3 (hence  C1,  C2 herein) and the corresponding indices were stated in terms of estimates for the bandit parameters, $\underline{\hat{\theta}}^i(t)$ an estimate of the parameters $\underline{\theta}^i$ of bandit $i$, given $t$ samples. In particular, Eq. \eqref{eqn:bk-index} appears essentially as
\begin{equation}
\ui(n,t)=\aui(n, t,\underline{\hat{\theta}}^i(t))=\sup_{ \underline{\theta}' \in \Theta}\left\{\mu(\underline{\theta}' ) \ : \   \bI( f_{ \underline{\hat{\theta}}^i(t) }, f_{ \underline{\theta}' }) < \frac{\ln n  }{t}\right\} .
\end{equation}
%
%
%

\end{remark}
%
%

 Previous work in this area includes   \cite{Rb52}, and additionally  \cite{gittins-79},   \cite{lai85} and \cite{weber1992gittins} there is a large literature on versions of this problem, cf. \cc{burnetas2003asymptotic}, \cc{burnetas1997finite} and references therein.  For  recent work in this area we refer to   \cite{audibert2009exploration}, 
\cite{auer2010ucb}, \cite{gittins2011multi}, \cite{bubeck2012best},  
\cite{cappe2013kullback}, 
\cite{kaufmann14}, 
\cite{2014minimax},  
 \cite{cowan15s}, \cite{cowan2015multi}, 
and references therein. 
    For more general dynamic programming extensions 
we refer to  
\cite{bkmdp97}, \cite{butenko2003cooperative}, \cite{optimistic-mdp}, \cite{audibert2009exploration}, \cite{littman2012inducing}, \cite{feinberg2014convergence} and references therein. 
To our knowledge, outside the work in \cite{lai85}, 
  \cc{bkmab96} and \cite{bkmdp97},  asymptotically optimal policies have only been developed in \cite{honda13} for the problem discussed herein and 
 in 
 \cc{honda2011asymptotically} and \cc{honda2010}  for the   
  problem of finite known support where  
optimal policies, cyclic and randomized, that are simpler to implement than those consider in   \cc{bkmab96} were constructed.
    Other related work in this area includes: \cite{Kat86}, \cite{Kat87}, \cite{burnetas1993sequencing}, \cite{BKlarge1996}, \cite{lagoudakis2003least},
\cite{bartlett2009regal}, \cite{tekin2012approximately}, \cite{jouini2009multi},
 \cite{dayanik2013asymptotically}, \cite{filippi2010optimism}, \cite{osband2014near},
 \cc{bkmdp97},
 \cc{aecd2014generalised},  \cc{cd2012robust}.
 
 \chapter{Optimal UCB Policies for Uniform Distributions}

 \section{The B-K Lower Bounds and Inflation Factors}
In this section we  take  $\mathcal{F}$ as the set of probability densities  on $\boldmath{R}$ uniform over {\sl some finite interval}, taking $f \in \mathcal{F}$ as uniform over $[a_f, b_f].$ Note, as the family of densities is parametrizable, this largely falls under the scope of \cite{bkmab96}. However, the results to follow seem to demonstrate a hole in that general treatment of the problem.

Note, some care with respect to support must be taken in applying \cite{bkmab96} to this case, to ensure that the integrals remain well defined. But for this $\mathcal{F}$, we have that for a given $f \in \mathcal{F}$, for any $g \in \mathcal{F}$ such that $\supp(f) \subset \supp(g)$, i.e.,   $a_g \leq a_f$ and $b_f \leq b_g$, 
\begin{equation}
\bI(f, g) =  \bE_f \left[ \ln \left( \frac{f(X)}{g(X)} \right) \right] = 
\ln \left( \frac{ b_g - a_g }{ b_f - a_f } \right).
\end{equation}
If $\supp(f)$ is not a subset of $\supp(g)$, we take $\bI(f, g)$ as infinite.
 
For notational convenience, given $\{ f_i \} \subset \mathcal{F}$, for each $i$, we take $f_i \in \mathcal{F}$ as supported on some  interval $[a_i, b_i]$. Note then, $\mu_i = (a_i + b_i)/2$.

Given $t$ samples from bandit $i$, $\{ X^i_{t'} \}_{t' = 1}^t$, we take
\begin{equation}
\begin{split}
\hat{a}^i_t & = X^i_{(1;t)}=\min_{t' \leq  t} X^i_{t'},\\
\hat{b}^i_t & = X^i_{(t;t)} = \max_{t' \leq  t} X^i_{t'},\\
\end{split}
\end{equation}
as the maximum-likelihood estimators of $a_i$ and $b_i$ respectively. We may then define $\hat{f}^i_t \in \mathcal{F}$ as the uniform density over the interval $[\hat{a}^i_t, \hat{b}^i_t]$. Note, $\hat{f}^i_t$ is the maximum-likelihood estimate of $f_i$.
%


We can now state and prove the following. 

\begin{lemma} Under Assumption 1 the following are true.
\begin{equation}\label{en:lower-b-m}
  \bM_{\ttext{BK}}( \{ f_i \} )=\sum_{i:\mu_i \neq \mu^*} \frac{ \Delta_i }{ \ln \left( 1 + \frac{2 \Delta_i }{ b_i - a_i } \right) }  .
\end{equation}
 %
\begin{equation}\label{en:lower-u-m}
\aui(n, t, \hat{f}^i_t)= \hat{a}^i_t + \frac{1}{2}\left( \hat{b}^i_t - \hat{a}^i_t \right) n^{1/t}  .
\end{equation}
\end{lemma}
\begin{proof}
Eq. (\ref{en:lower-b-m}) follows from Eq. (\ref{eqn:general-boundm}) and the 
observation that in this case:
$$
\inf_{g \in \mathcal{F}} \{ \bI(f_i, g) : \mu(g) \geq \mu^* \} =  \ln \left( \frac{2 \mu^* - 2 a_i }{ b_i - a_i } \right) = \ln \left( 1 + \frac{2 \mu^* - 2 \mu_i }{ b_i - a_i } \right) .
$$
For Eq. (\ref{en:lower-u-m}) we have:
\begin{equation}
\begin{split}
\aui(n, t, \hat{f}^i_t) & = \sup_{g \in \mathcal{F}}\left\{\mu(g) \ : \   \bI(\hat{f}^i_t,g) < \frac{\ln n  }{t}\right\}\\
& = \sup_{a \leq \hat{a}^i_t, b \geq \hat{b}^i_t} \left\{ \frac{a + b}{2} \ : \   \ln \left( \frac{ b - a }{  \hat{b}^i_t  -  \hat{a}^i_t } \right) < \frac{\ln n  }{t}\right\}\\
& = \sup_{a \leq \hat{a}^i_t, b \geq \hat{b}^i_t} \left\{ a + \frac{b - a}{2} \ : \     (b - a)  < (\hat{b}^i_t  -  \hat{a}^i_t)n^{1/t} \right\}\\
& = \hat{a}^i_t + \frac{1}{2}\left( \hat{b}^i_t - \hat{a}^i_t \right) n^{1/t}.
\end{split}
\end{equation}

%
\end{proof}

 We are interested in policies $\pi$ such that $\lim_n R_\pi(n)/\ln n$ achieves the lower bound indicated above, for every choice of $\{ f_i \} \subset \mathcal{F}$. 
Following the prescription of \cite{bkmab96}, i.e. Eq. (\ref{en:lower-u-m}),  would lead to the following policy,

\textbf{Policy BK-UCB :  $ \boldmath \api $}. At each $n=1,2,\ldots$:

i) For $n=1,2,\ldots,2N,$ sample each bandit  twice, and 

ii)   for $n \geq 2N$, let $\api (n+1)$  be equal   to:
\begin{equation}
 \argmax_i \left\{ \hat{a}^i_{T^i_{\api}(n)}   + \frac{1}{2}\left(\hat{b}^i_{T^i_{\api}(n)} -\hat{a}^i_{T^i_{\api}(n)} \right)     n^{\frac{1}{ T^i_{\api}(n)}}  \  \right\},
\end{equation}
 breaking ties arbitrarily.

It is easy to demonstrate that the estimators $\hat{\underline{\theta}}^i(t) = (\hat{a}^i_t, \hat{b}^i_t)$ converge sufficiently quickly to $(a_i, b_i)$ in probability that Condition (C1) above is satisfied for $\hat{f}^i_t$. Proving that Condition (C2) is satisfied, however, is much much more difficult, and in fact we conjecture that  (C2) does \emph{not} hold for policy \api. While this does not indicate that that $\api$ fails to achieve asymptotic optimality, it does imply that the standard techniques are insufficient to verify it. However, asymptotic optimality may provably be achieved by an (seemingly) negligible modification, via the following policy.

\section{Asymptotically Optimal UCB Policy}
We propose the following policy:

\textbf{Policy UCB-Uniform: $ \boldmath \cpi $}. At each $n=1,2,\ldots$:

i)  For $n=1,2,\ldots,3N$ sample each bandit three times, and

ii)   for $n \geq 3N$, let $\cpi (n+1)$  be equal   to:
\begin{equation}
 \argmax_i \left\{ \hat{a}^i_{T^i_{\cpi}(n)}   + \frac{1}{2}\left(\hat{b}^i_{T^i_{\cpi}(n)} -\hat{a}^i_{T^i_{\cpi}(n)} \right)     n^{\frac{1}{ T^i_{\cpi}(n) - 2}}  \  \right\},
\end{equation}
%
 breaking ties arbitrarily.

In the remainder of this paper, we verify the asymptotic optimality of $\cpi$ (Theorem \ref{thm:optimality}), and additionally give finite horizon bounds on the regret under this policy (Theorem \ref{thm:finite-time}, \ref{thm:epsilon-free}). Further, while Theorem \ref{thm:epsilon-free} bounds the order of the remainder term as $O((\ln n)^{3/4})$, this is refined somewhat in Theorem \ref{thm:remainder-refined} to $o( (\ln n)^{2/3 + \beta} )$.

\chapter{The Optimality Theorem and Finite Time Bounds}

For the work in this section it is convenient to define the bandit spans, $S_i = b_i - a_i$. We take $S_*$ to be the minimal span of any optimal bandit, i.e., 
 $$S_* = \min_{i:\mu_i = \mu^*} S^i .$$

Recall that $\Delta_i = \mu^* - \mu_i =\max_{j}\{\frac{a_j+b_j}{2}\} -\frac{a_i+b_i}{2} $.   The primary result of this paper is the following.

\begin{theorem}\label{thm:finite-time}
For each sub-optimal $i$ (i.e., $\mu_i \neq \mu^*$), let $(\epsilon_i, \delta_i)$ be such that $0 < \epsilon_i < S_*$, $0 < \delta_i < S_i$, and $\epsilon_i + \delta_i < \Delta_i$. For \cpi\ as defined above, for all $n \geq 3N$:
\begin{equation}
\begin{split}
R_{\cpi}(n) \leq & \left( \sum_{i:\mu_i \neq \mu^*}  \frac{ \Delta_i  }{ \ln \left( 1 + \frac{ 2\Delta_i}{S_i} \left(1- \frac{ (\epsilon_i +\delta_i) }{\Delta_i} \right)\right) }  \right) \ln n \\
& + \sum_{i:\mu_i \neq \mu^*} \left( \frac{ S_i }{\delta_i} + \frac{3}{8} \frac{S^3_*}{\epsilon_i^3} + 18 \right) \Delta_i.
\end{split}
\end{equation}

\end{theorem}

The proof of Theorem \ref{thm:finite-time} is the central proof of this paper. We delay it briefly, to present two related results that can be derived from the above. The first is that \cpi\ is asymptotically optimal.

\begin{theorem}\label{thm:optimality}
For \cpi\ as defined above, \cpi\ is asymptotically optimal in the sense that
\begin{equation}
\lim_n \frac{ R_{\cpi}(n) }{ \ln n } =  \bM_{\ttext{BK}}( \{ f_i \} )= \sum_{i:\mu_i \neq \mu^*} \frac{ \Delta_i }{ \ln \left( 1 + \frac{2 \Delta_i }{ S_i } \right) }.
\end{equation}
\end{theorem}

\begin{proof}
Fix the $(\epsilon_i, \delta_i)$ as feasible in the hypotheses of Theorem \ref{thm:finite-time}. In that case, we have
\begin{equation}
\limsup_n \frac{ R_{\cpi}(n) }{ \ln n } \leq \sum_{i:\mu_i \neq \mu^*}  \frac{ \Delta_i }{ \ln \left( 1 + \frac{ 2\Delta_i}{S_i}\left(1 - \frac{(\epsilon_i +\delta_i)}{\Delta_i}\right) \right) }.
\end{equation}
Taking the infimum as $\epsilon_i + \delta_i \to 0$ yields
\begin{equation}
\limsup_n \frac{ R_{\cpi}(n) }{ \ln n } \leq \sum_{i:\mu_i \neq \mu^*}  \frac{ \Delta_i }{ \ln \left( 1 + \frac{ 2\Delta_i}{S_i} \right) }.
\end{equation}
This, combined with the previous observation about the $\liminf$ in Eq. (\ref{en:lower-b-m}) completes the result.
\end{proof}

We next give an `$\epsilon$-free' version of the previous bound, which demonstrates the remainder term on the regret under \cpi \ is at worst $O((\ln n)^{3/4})$.

\begin{theorem}\label{thm:epsilon-free}
For each sub-optimal $i$ (i.e., $\mu_i \neq \mu^*)$, let $G_i = \min\left( S_*, S_i, \frac{1}{4} \Delta_i \right)$. For all $n \geq 3N$,
\begin{equation}
\begin{split}
& R_{\cpi}(n) \leq \left( \sum_{i:\mu_i \neq \mu^*} \frac{ \Delta_i}{ \ln \left( 1 + \frac{2 \Delta_i }{ S_i } \right) } \right) ( \ln n )\\
 & + \sum_{i:\mu_i \neq \mu^*} \left( \frac{8 G_i \Delta_i }{ (S_i+2\Delta_i) \ln\left(1+\frac{ 2\Delta_i}{S_i}\right)^2} + \frac{3S^3_* \Delta_i}{8G^3_i} \right) ( \ln n )^{3/4} \\
& +  \sum_{i:\mu_i \neq \mu^*} \left(\frac{ S_i \Delta_i }{G_i} \right) ( \ln n )^{1/4}  + 18  \sum_{i:\mu_i \neq \mu^*} \Delta_i.
\end{split}
\end{equation}
\end{theorem}

\begin{proof}[Proof of Theorem \ref{thm:epsilon-free}]
Let $0 < \epsilon < 1$, and for each $i$ let $\epsilon_i = \delta_i = G_i \epsilon$. Hence,
\begin{equation}
\ln \left( 1 + \frac{ 2\Delta_i}{S_i} \left(1- \frac{ (\epsilon_i +\delta_i) }{\Delta_i} \right)\right) = \ln \left( 1 + \frac{ 2\Delta_i}{S_i} \left(1- \epsilon \frac{2 G_i}{\Delta_i} \right)\right).
\end{equation}
Define
\begin{equation}
D_i = \frac{ 1 }{ \ln \left( 1 + \frac{ 2\Delta_i}{S_i} \left(1- \epsilon \frac{2 G_i}{\Delta_i} \right)\right) } -  \frac{1}{\ln\left(1 + \frac{ 2\Delta_i}{S_i}\right)}.
\end{equation}

 Note the following bound, that 
\begin{equation}
\begin{split}
D^i & \leq \left( \frac{2 G_i \epsilon }{\Delta_i-2 G_i \epsilon} \right) \frac{2\Delta_i}{ (S_i+2\Delta_i) \ln\left(1+\frac{ 2\Delta_i}{S_i}\right)^2} \\
& \leq \left( \frac{2 G_i \epsilon }{\frac{1}{2} \Delta_i } \right) \frac{2\Delta_i}{ (S_i+2\Delta_i) \ln\left(1+\frac{ 2\Delta_i}{S_i}\right)^2} \\
& =\frac{8 G_i \epsilon}{ (S_i+2\Delta_i) \ln\left(1+\frac{ 2\Delta_i}{S_i}\right)^2}
\end{split}
\end{equation}
This first inequality is proven separately as Proposition \ref{prop:log-bound} in the Appendix. The second inequality is simply the observation that $2 G_i \epsilon \leq 2 G_i \leq \frac{1}{2} \Delta_i$. Applying this bound to Theorem \ref{thm:finite-time} yields the following bound,
\begin{equation}
\begin{split}
R_{\cpi}(n) \leq & \left(  \sum_{i:\mu_i \neq \mu^*} \frac{ \Delta_i}{ \ln \left( 1 + \frac{2 \Delta_i }{ S_i } \right) } \right) (\ln n)\\
  & + 8 \left( \sum_{i:\mu_i \neq \mu^*} \frac{G_i \Delta_i }{ (S_i+2\Delta_i) \ln\left(1+\frac{ 2\Delta_i}{S_i}\right)^2} \right) \epsilon \ln n \\
& + \left( \sum_{i:\mu_i \neq \mu^*} \frac{ S_i \Delta_i }{G_i} \right) \epsilon^{-1} +  \frac{3}{8} S^3_* \left(\sum_{i:\mu_i \neq \mu^*} \frac{\Delta_i }{G^3_i} \right) \epsilon^{-3} \\
& + 18 \left( \sum_{i:\mu_i \neq \mu^*} \Delta_i \right).
\end{split}
\end{equation}
Taking $\epsilon = (\ln n)^{-1/4}$ completes the proof.
\end{proof}

\begin{proof}[Proof of Theorem 1]
For any $i$ such that $\mu_i \neq \mu^*$, recall that bandit $i$ is taken to be uniformly distributed on the interval $[a_i, b_i]$.  Let $(\epsilon_i, \delta_i)$ be as hypothesized. In this proof, we take $\pi = \cpi$ as defined above. Additionally, for each $i$ we define $W^i_k = \max_{t \leq k} X^i_t$ and $V^i_k = \min_{t \leq k} X^i_t$. We define the index function
\begin{equation}
u_i(k, j) = V^i_j + \frac{1}{2}( W^i_j - V^i_j ) k^{\frac{1}{j-2}}.
\end{equation}

We define the following events of interest, $\mathcal{J}^i_t = \{ u_i(t, T^i_\pi(t)) \geq \mu^* -\epsilon_i \}$ and $\mathcal{K}^i_s =  \{ V^i_{s} \leq a_i + \delta_i \}$. We now define the following quantities: For $n \geq 3N$,
\begin{equation}
\begin{split}
n^i_1(n, \epsilon_i, \delta_i) & = \sum_{t = 3N}^n \bo \{ \pi(t+1) = i, \mathcal{J}^i_t, \mathcal{K}^i_{T^i_\pi(t)} \} \\
n^i_2(n, \epsilon_i, \delta_i) & = \sum_{t = 3N}^n \bo \{ \pi(t+1) = i, \mathcal{J}^i_t, \overline{\mathcal{K}^i_{T^i_\pi(t)}} \} \\
n^i_3(n, \epsilon_i, \delta_i) & = \sum_{t = 3N}^n \bo \{ \pi(t+1) = i, \overline{\mathcal{J}^i_t}\}.
\end{split}
\end{equation}
Hence, we have the following relationship for $n \geq 3N$, that
\begin{equation}\label{eqn:t-n-relation}
\begin{split}
T^i_\pi(n+1) & = 3 + \sum_{t = 3N}^n \bo  \{ \pi(t+1) = i \} \\
& = 3 + n^i_1(n, \epsilon_i, \delta_i) + n^i_2(n, \epsilon_i, \delta_i) + n^i_3(n, \epsilon_i, \delta_i).
\end{split}
\end{equation}
The proof proceeds by bounding, in expectation, each of the three terms.

Observe that, by the structure of the index function $u_i$, 
\begin{equation}
\begin{split}
& \bo \{ \pi(t+1) = i, \mathcal{J}^i_t, \mathcal{K}^i_{ T^i_\pi(t) } \}\\
 & \leq \bo \left\{ \pi(t+1) = i, a_i + \delta_i + \frac{1}{2}( b_i - a_i ) t^{\frac{1}{T^i_\pi(t)-2}} \geq \mu^* - \epsilon_i \right\} \\
 & = \bo \left\{ \pi(t+1) = i, T^i_\pi(t) \leq \frac{ \ln t }{ \ln \left( \frac{ 2\mu^*-2a_i - 2\epsilon_i - 2\delta_i }{b_i - a_i} \right) } + 2 \right\}\\
 & \leq \bo \left\{ \pi(t+1) = i, T^i_\pi(t) \leq \frac{ \ln n }{ \ln \left( \frac{ 2\mu^*-2a_i - 2\epsilon_i - 2\delta_i }{b_i - a_i} \right) } + 2 \right\}.
\end{split}
\end{equation}
Hence,
\begin{equation}
\begin{split}
& n^i_1(n, \epsilon_i, \delta_i) \leq \\
& \sum_{t = 3N}^n \bo \left\{ \pi(t+1) = i, T^i_\pi(t) \leq \frac{ \ln n }{ \ln \left( \frac{ 2\mu^*-2a_i - 2\epsilon_i - 2\delta_i }{b_i - a_i} \right) } + 2 \right\} \\
& \leq \sum_{t = 1}^n \bo \left\{ \pi(t+1) = i, T^i_\pi(t) \leq \frac{ \ln n }{ \ln \left( \frac{ 2\mu^*-2a_i - 2\epsilon_i - 2\delta_i }{b_i - a_i} \right) } + 2 \right\} \\
& \leq  \frac{ \ln n }{ \ln \left( \frac{ 2\mu^*-2a_i - 2\epsilon_i - 2\delta_i }{b_i - a_i} \right) } + 2 + 2.
\end{split}
\end{equation}
The last inequality follows, observing that $T^i_\pi(t)$ may be expressed as the sum of $\pi(t) = i$ indicators, and seeing that the additional condition bounds the number of non-zero terms in the above sum. The additional $+2$ simply accounts for the $\pi(1) = i$ term and the $\pi(n+1) = i$ term. 

 Note, this bound is sample-path-wise.


For the second term,
\begin{equation}\label{en:second-term}
\begin{split}
n^i_2(n, \epsilon_i, \delta_i) & \leq \sum_{t = 3N}^n \bo \{ \pi(t+1) = i,  \overline{\mathcal{K}^i_{T^i_\pi(t)}} \} \\
& =  \sum_{t = 3N}^n \sum_{k = 2}^t \bo \{ \pi(t+1) = i, \overline{\mathcal{K}^i_k}, T^i_\pi(t) = k \} \\
& =  \sum_{t = 3N}^n \sum_{k = 2}^t \bo \{ \pi(t+1) = i, T^i_\pi(t) = k \}\bo \{ \overline{\mathcal{K}^i_k} \} \\
& \leq \sum_{k = 2}^n \bo \{ \overline{\mathcal{K}^i_k} \} \sum_{t = k}^n \bo \{ \pi(t+1) = i, T^i_\pi(t) = k \} \\
& \leq \sum_{k = 2}^n \bo \{ \overline{\mathcal{K}^i_k} \}\\
& = \sum_{k = 2}^n \bo \{ V^i_k > a_i + \delta_i \}.
\end{split}
\end{equation}
The last inequality follows as, for fixed $k$, $\{ \pi(t+1) = i, T^i_\pi(t) = k \}$ may be true for at most one value of $t$. It follows then that
\begin{equation}
\begin{split}
\bE \left[ n^i_2(n, \epsilon_i, \delta_i) \right] & \leq \sum_{k = 2}^n \bP \left( V^i_k > a_i + \delta_i \right) \\
& = \sum_{k = 2}^n \bP \left( X^i_1 > a_i + \delta_i \right)^k \\
& = \sum_{k = 2}^n \left( 1 - \frac{ \delta_i }{ b_i - a_i } \right)^k \\
& \leq \sum_{k = 1}^\infty \left( 1 - \frac{ \delta_i }{ b_i - a_i } \right)^k = \frac{ b_i - a_i }{\delta_i} - 1 < \infty. \\
\end{split}
\end{equation}

To bound the $n^i_3$ term, observe that in the event $\pi(t+1) = i$, from the structure of the policy it must be true that $u_i(t, T^i_\pi(t)) = \max_j u_j(t, T^j_\pi(t))$. Thus, if $i^*$ is some bandit such that $\mu_{i^*} = \mu^*$, $u_{i^*}(t, T^{i^*}_\pi(t) ) \leq u_i(t, T^i_\pi(t))$. In particular, we take $i^*$ to be the optimal bandit realizing the minimal span $b_{i^*} - a_{i^*}$. It follows,
\begin{equation}
\begin{split}
n^i_3(n, \epsilon_i, \delta_i) & \leq \sum_{t = 3N}^n \bo \{ \pi(t+1) = i, u_{i^*}(t, T^{i^*}_\pi(t)) < \mu^* - \epsilon_i \} \\
& \leq \sum_{t = 3N}^n \bo \{ u_{i^*}(t, T^{i^*}_\pi(t)) < \mu^* -  \epsilon_i \} \\
& \leq \sum_{t = 3N}^n \bo \{ u_{i^*}(t, s) < \mu^* - \epsilon_i \text{ for some } 3 \leq s \leq t\}.
\end{split}
\end{equation}
The last step follows as for $t$ in this range, $3 \leq T^{i^*}_\pi(t) \leq t$. Hence
\begin{equation}\label{eqn:n-4-bound}
\begin{split}
& \bE \left[ n^i_3(n, \epsilon_i, \delta_i) \right] \\
& \leq \sum_{t = 3N}^n \bP \left( u_{i^*}(t, s) < \mu^* - \epsilon_i \text{ for some } 3 \leq s \leq t \right) \\
& \leq \sum_{t = 3N}^n \sum_{s = 3}^t \bP \left( u_{i^*}(t, s) < \mu^* - \epsilon_i \right).
\end{split}
\end{equation}
Here we may make use of the following result:
\begin{lemma}
Let $X_1, X_2, \ldots$ be i.i.d. $\text{Unif}[a,b]$ random variables, with $a < b$, $a$ and $b$ finite. For $k \geq 2$, let $W_k = \max_{t \leq k} X_t$ and $V_k = \min_{t \leq k} X_t$. In that case, the joint density of $(W_k, V_k)$ is given by:
\begin{equation}
f_k(w,v) = \begin{cases} k(k-1)(b-a)^{-k}(w-v)^{k-2}  &\mbox{if } v \leq w \\ 
0 & \mbox{ else.} \end{cases}
\end{equation}
\end{lemma}
We therefore have that
\begin{equation}
\begin{split}
& \bP \left( u_{i^*}(t, s) < \mu^* - \epsilon_i \right)\\
 & = \bP \left( V^{i^*}_s + \frac{1}{2}\left( W^{i^*}_s - V^{i^*}_s \right) t^{1/(s-2)} < \mu^* - \epsilon_i \right) \\
& = \int_{a_{i^*}}^{\mu^* - \epsilon_i} \int_{v}^{\min\left(b_{i^*}, v + 2 \frac{ (\mu^* - \epsilon_i) - v }{t^{1/(s-2)}} \right)} f_s(w,v) dw dv \\
& \leq \int_{a_{i^*}}^{\mu^* - \epsilon_i} \int_{v}^{v + 2 \frac{ (\mu^* - \epsilon_i) - v }{t^{1/(s-2)}}} f_s(w,v) dw dv \\
& = \frac{1}{2} t^{-\frac{(s-1)}{(s-2)}} \left( 2 \frac{(\mu^* - \epsilon_i) - a_{i^*}}{ b_{i^*} - a_{i^*} }\right)^s \\
& = \frac{1}{2} t^{-1} t^{-1/(s-2)} \left( 1 - \frac{2\epsilon_i}{b_{i^*} - a_{i^*}} \right)^s.
\end{split}
\end{equation}
The last step is simply the observation that $\mu^* = (a_{i^*} + b_{i^*})/2$. For convenience, let $\alpha = 2\epsilon_i/(b_{i^*} - a_{i^*})$. We therefore have that
\begin{equation}
\begin{split}
\sum_{s = 3}^t \bP \left( u_{i^*}(t, s) < \mu^* - \epsilon \right) &  \leq \sum_{s = 3}^t  \frac{1}{2} t^{-1} t^{-1/(s-2)} (1-\alpha)^s \\
&  \leq \sum_{s = 1}^{t-2}  \frac{1}{2} t^{-1} t^{-1/s} (1-\alpha)^{s+2} \\
& \leq \frac{1}{2} t^{-1} (1-\alpha)^2 \sum_{s = 1}^\infty t^{-1/s} (1-\alpha)^s.
\end{split}
\end{equation}
Hence, from Eq. \eqref{eqn:n-4-bound} and the above,
\begin{equation}
\begin{split}
\bE \left[ n^i_3(n, \epsilon_i, \delta_i) \right] & \leq \sum_{t = 6}^n \frac{1}{2} t^{-1} (1-\alpha)^2 \sum_{s = 1}^\infty t^{-1/s} (1-\alpha)^s \\
& \leq \frac{1}{2} (1-\alpha)^2 \sum_{t = 6}^n t^{-1} \sum_{s = 1}^\infty t^{-1/s} (1-\alpha)^s \\
& \leq (1 - \alpha)^2 \left( 15 + \frac{3}{\alpha^3} \right).
\end{split}
\end{equation}
The last step is a bound proved separately as Proposition \ref{prop:summation-bound} in the Appendix. Observing further that $1-\alpha \leq 1$, we have finally that
\begin{equation}
\bE \left[ n^i_3(n, \epsilon_i, \delta_i) \right] \leq 15 + \frac{3}{\alpha^3} = 15 + \frac{3}{8} \frac{(b_{i^*} - a_{i^*})^3}{\epsilon_i^3}.
\end{equation}
Observing that $T^i_\pi(n) \leq T^i_\pi(n+1)$, bringing the three terms together we have that
\begin{equation}
\bE \left[ T^i_\pi(n) \right] \leq \frac{ \ln n }{ \ln \left( \frac{ 2\mu^*-2a_i - 2\epsilon_i - 2\delta_i }{b_i - a_i} \right) }  + \frac{ b_i - a_i }{\delta_i} + \frac{3}{8} \frac{(b_{i^*} - a_{i^*})^3}{\epsilon_i^3} + 18.
\end{equation}
The result then follows from the definition of regret, Eq. \eqref{eqn:regret}, and the observation again that $\mu_i = (b_i + a_i)/2$.
\end{proof}

At various points in the results so far, choices of convenience were made with the purpose of keeping associated constants and coefficients `nice'. The techniques and results above may actually be refined slightly to present a somewhat stronger result on the remainder term, at the cost of more complicated coefficients. In particular,

\begin{theorem}\label{thm:remainder-refined}
For any $\beta > 0$,
\begin{equation}
R_{\cpi}(n) \leq \sum_{i:\mu_i \neq \mu^*} \frac{ \Delta_i \ln n }{ \ln \left( 1 + \frac{2 \Delta_i }{ S_i } \right) } + o(( \ln n )^{2/3 + \beta}).
\end{equation}
\end{theorem}
\begin{proof}
Note that, given the result of Theorem \ref{thm:epsilon-free}, it suffices to take $\beta \leq 1/12$.

Building on the proof of Theorem \ref{thm:finite-time}, taking $\alpha = 2\epsilon_i/(b_{i^*} - a_{i^*}) = 2\epsilon_i/S_*$ where $i^*$ is the optimal bandit that realizes the smallest value of $b_{i^*} - a_{i^*}$, we have that
\begin{equation}
\begin{split}
\bE \left[ T^i_\pi(n) \right] & \leq \frac{ \ln n }{ \ln \left( \frac{ 2\mu^*-2a_i - 2\epsilon_i - 2\delta_i }{b_i - a_i} \right) } + \frac{ b_i - a_i }{\delta_i} \\
& + \frac{1}{2} (1-\alpha)^2 \sum_{t = 6}^n t^{-1} \sum_{s = 1}^\infty t^{-1/s} (1-\alpha)^s + 3 \\
& \leq \frac{ \ln n }{ \ln \left( 1 + \frac{2\Delta_i}{ S_i } \left( 1 - \frac{\epsilon_i + \delta_i}{\Delta_i} \right ) \right) } + \frac{ S_i }{\delta_i} \\
& + \frac{1}{2} \sum_{t = 6}^n t^{-1} \sum_{s = 1}^\infty t^{-1/s} (1-\alpha)^s + 3. \\
\end{split}
\end{equation}
The proof of Theorem \ref{thm:finite-time} then proceeded to bound the above double sum using Proposition \ref{prop:summation-bound}. Utilizing the proof of Proposition \ref{prop:summation-bound} (but without choosing specific values of $p < 1$, $q > 1$ to render `nice' coefficients), we have
\begin{equation}
\begin{split}
& \sum_{t = 6}^n t^{-1} \sum_{s = 1}^\infty t^{-1/s} (1-\alpha)^s\\
 & \leq \left( \left( \frac{1}{e} \frac{p+q}{1-p} \right)^{ \frac{p+q}{1-p} } + \frac{1}{\alpha} \left( \frac{1}{e \alpha} \frac{q}{p} \right)^{ \frac{q}{p} } \right) \frac{1}{q-1}\\
& = \alpha^{-1-\frac{q}{p}} C_1(p, q) + C_2(p,q) \\
& = \epsilon_i^{-1-\frac{q}{p}} \left( \frac{ S_* }{2} \right)^{1+\frac{q}{p} } C_1(p, q) + C_2(p,q).
\end{split}
\end{equation}
Where for convenience we are defining $C_1, C_2$ as the associated functions of $p, q$. Note, they are finite for $p < 1$, $q > 1$. Let $0 < \epsilon < 1$ and define $G_i = \min\left( S_*, S_i, \frac{1}{4} \Delta_i \right)$ as in Theorem \ref{thm:epsilon-free}. Taking $\epsilon_i = \delta_i = \epsilon G_i$, we have the following bound (utilizing Proposition \ref{prop:log-bound} as in the proof of Theorem \ref{thm:epsilon-free}):
\begin{equation}
\begin{split}
\bE \left[ T^i_\pi(n) \right] \leq &  \frac{\ln n}{\ln\left(1 + \frac{ 2\Delta_i}{S_i}\right)} + \frac{8 G_i \epsilon \ln n}{ (S_i+2\Delta_i) \ln\left(1+\frac{ 2\Delta_i}{S_i}\right)^2} \\
& + \frac{S_i}{G_i} \epsilon^{-1}  +  \epsilon^{-1-\frac{q}{p}} \left( \frac{ S_* }{2 G_i} \right)^{1+\frac{q}{p} } C_1(p, q) \\
& + C_2(p,q) + 3.
\end{split}
\end{equation}
At this point, taking $\epsilon = (\ln n)^{-p/(2p + q)}$ yields the following
\begin{equation}
\begin{split}
\bE \left[ T^i_\pi(n) \right] \leq & \frac{\ln n}{\ln\left(1 + \frac{ 2\Delta_i}{S_i}\right)} + \frac{8 G_i (\ln n)^{ \frac{p+q}{2p + q} }}{ (S_i+2\Delta_i) \ln\left(1+\frac{ 2\Delta_i}{S_i}\right)^2} \\
& + \frac{S_i}{G_i} (\ln n)^{\frac{p}{2p+q}}  +  (\ln n)^{ \frac{p+q}{2p + q} } \left( \frac{ S_* }{2 G_i} \right)^{1+\frac{q}{p} } C_1(p, q)\\
& + C_2(p,q) + 3,
\end{split}
\end{equation}
or more conveniently,
\begin{equation}
\bE \left[ T^i_\pi(n) \right] \leq \frac{\ln n}{\ln\left(1 + \frac{ 2\Delta_i}{S_i}\right)} + O(  (\ln n)^{ \frac{p+q}{2p + q} } ) + O(  (\ln n)^{ \frac{p}{2p + q} } ).
\end{equation}
Taking $q = \gamma p$, where $(\gamma, p)$ is chosen such that $1/\gamma < p < 1$, the above yields (via the definition of regret, Eq. \eqref{eqn:regret}):
\begin{equation}
R_\pi(n) \leq \sum_{i:\mu_i \neq \mu^*} \frac{ \Delta_i \ln n}{\ln\left(1 + \frac{ 2\Delta_i}{S_i}\right)} + O(  (\ln n)^{ \frac{1+\gamma}{2 + \gamma} } ) + O(  (\ln n)^{ \frac{1}{2 + \gamma} } ).
\end{equation}
At this point, note that taking $\gamma = 2$ recovers the remainder order given in Theorem \ref{thm:epsilon-free}. For a given $1/12 > \beta > 0$, taking $\gamma < (1 + 6 \beta)/(1 - 3\beta)$ yields $(1 + \gamma)/(2 + \gamma) < 2/3 + \beta$, and completes the proof.
\end{proof}

\section{Simulation Comparisons of the \upi\   Sampling}
 In order to obtain a picture  of the benefits of the \upi\ sampling policy, we compared it with the best known alternatives. In both figures below, curve ($\ell$) ($\ell=1,2,3$) is a plot of the average (over $20,000$ repetitions in Fig. 1 and  $10,000$ repetitions in Fig. 2) regret of sampling using  policies \upi,  \krpi , and \cpi,  respectively; where policy
  \krpi , is based on the sampling policy in 
  \cite{rmab1995}, and 
  \cpi\ is a recently shown, cf. \cc{chk2015},  asymptotically 
 optimal policy for the case in which the population outcomes 
  distributions are normal with unknown means and unknown variances. 
  Specifically, 
  given $t$ samples from bandit $i$ at round (global time) $n$,
 \upi\ \cpi\ and \krpi\   are maximum index based policies with indices 
$\cui(n, t)$, $\ui_{\text{CHK}}$, and  $\ui_{\text{KR}}$  
where the first is defined by Eq. (16) and   the other two are given by: 
$\ui_{\text{CHK}} (n,t)= \bar{X}^i_{t} +     S^i(t) \sqrt{ n^{\frac{2}{t - 2}}-1}  $ and
$\ui_{\text{KR}} (n,t)= \bar{X}^i_{t} +     S^i(t)\sqrt{ \frac{2\ln n}{t}}\, ,$ 
 where  $S^i(t) = (\sum_{k = 1}^t \left(X^i_k - \bar{X}^i_t\right)^2/t)^{1/2}$.

 \begin{figure}[h!]
\begin{center}
\includegraphics[width=.75\textwidth]{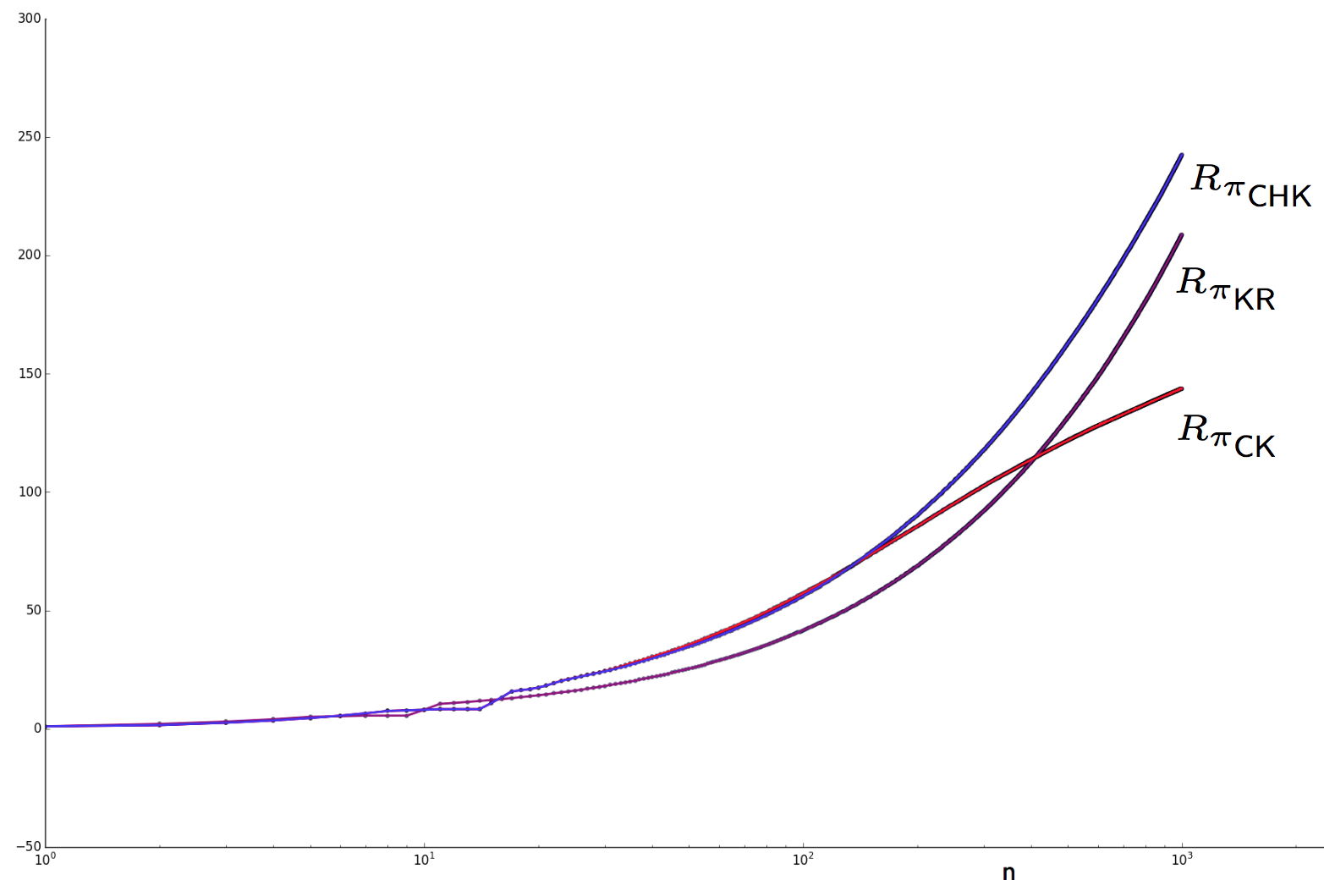}
\caption{Short Time Horizon:\small Numerical regret comparison of \upi , \krpi  , and \cpi ,   for the 6 bandits with parameters given in  Table 1.
Average values over $20,000$ repetitions.
} 
\end{center}
\end{figure}
  \begin{center} 
{ \small
 \begin{tabular}{|l||r|r|r|r|r|r|}
  \hline
   $i $ & 1 & 2 & 3 & 4 & 5 & 6\\
  \hline
   $a_i $ & 0 & 0 & 0 & 1 & 1 & 1\\
       \hline
$b_i$  & 10 & 9 & 8 & 9.5 & 10 & 5\\
  \hline  \hline    
  \multicolumn{7}{l}{Table  1}    \\
\end{tabular}
}
 \end{center}
 
  \begin{figure}[h!]
\begin{center}
\includegraphics[width=.75\textwidth]{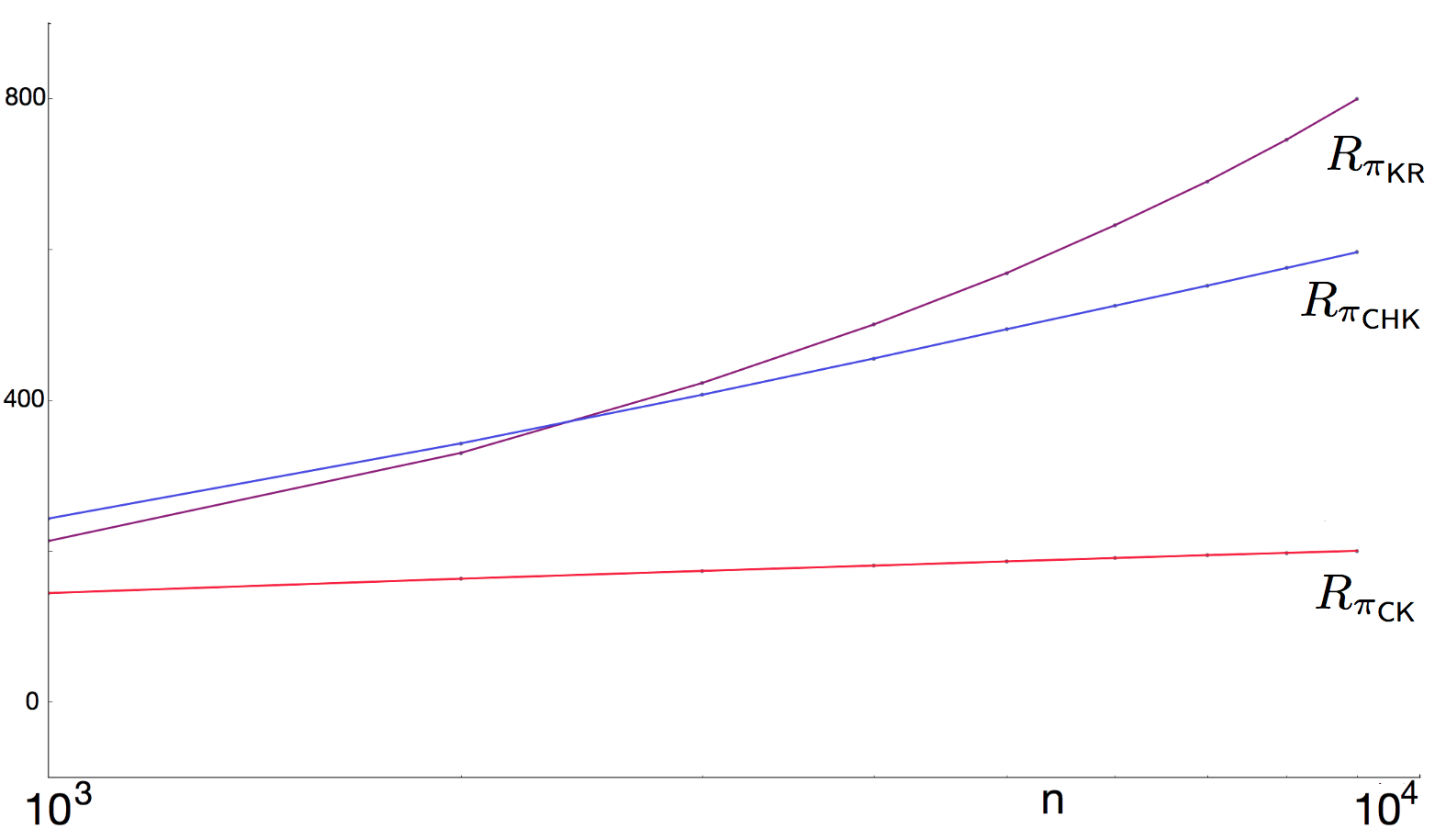}
\caption{Log Time Horizon:\small Numerical regret comparison of  \upi , \krpi  , and \cpi ,  for the 6 bandits with parameters given in  Table 1.
Average values over $10,000$ repetitions.
} 
\end{center}
\end{figure}
These graphs clearly illustrate the benefit of using the optimal policy.

\subsection{Acknowledgments.}
PhD student Daniel Pirutinsky did the simulation work underlying  Figures 1 and 2. Support for this project was provided by 
 the National Science Foundation (NSF grant CMMI-14-50743).

%


\subsection{Additional Proofs}

\begin{proposition}\label{prop:summation-bound}
For $0 < \alpha < 1$, for all $n \geq 3$,
\begin{equation}
\sum_{t = 3}^n \frac{1}{t} \sum_{s = 1}^\infty t^{-1/s} (1-\alpha)^s \leq 30 + \frac{6}{\alpha^3}.
\end{equation}
\end{proposition}
\begin{proof}[Proof of Proposition \ref{prop:summation-bound}]
Let $1 > p > 0$. We have
\begin{equation}
\begin{split}
& \sum_{s = 1}^\infty t^{-1/s} (1-\alpha)^s \\
& = \sum_{s = 1}^{\left \lfloor{ \ln(t)^p }\right \rfloor } t^{-1/s} (1-\alpha)^s + \sum_{s = \left \lceil{ \ln(t)^p }\right \rceil}^\infty t^{-1/s} (1-\alpha)^s \\
& \leq \sum_{s = 1}^{\left \lfloor{ \ln(t)^p }\right \rfloor } t^{-1/s} + \sum_{s = \left \lceil{ \ln(t)^p }\right \rceil}^\infty (1-\alpha)^s \\
& \leq \left \lfloor{ \ln(t)^p }\right \rfloor t^{-1/\left \lfloor{ \ln(t)^p }\right \rfloor} + \frac{1}{\alpha}(1-\alpha)^{ \left \lceil{ \ln(t)^p }\right \rceil } \\
& \leq  \ln(t)^p  t^{-1/ \ln(t)^p } + \frac{1}{\alpha}(1-\alpha)^{  \ln(t)^p } \\
& =  \ln(t)^p  e^{-\ln(t)^{1-p} } + \frac{1}{\alpha}(1-\alpha)^{  \ln(t)^p }.
\end{split}
\end{equation}
Here we may make use of the following bounds, that for $x \geq 0$, $q > 0$,
\begin{equation}
\begin{split}
x^p e^{-x^{1-p}} & \leq \left( \frac{1}{e} \frac{p+q}{1-p} \right)^{ \frac{p+q}{1-p} } x^{-q} \\
(1-\alpha)^{x^p} & \leq \left( \frac{-1}{e \ln(1-\alpha)} \frac{q}{p} \right)^{ \frac{q}{p} } x^{-q} \leq \left( \frac{1}{e \alpha} \frac{q}{p} \right)^{ \frac{q}{p} } x^{-q}.
\end{split}
\end{equation}
Applying these to the above,
\begin{equation}
\sum_{s = 1}^\infty t^{-1/s} (1-\alpha)^s \leq \left( \left( \frac{1}{e} \frac{p+q}{1-p} \right)^{ \frac{p+q}{1-p} } + \frac{1}{\alpha}\left( \frac{1}{e \alpha} \frac{q}{p} \right)^{ \frac{q}{p} } \right) \ln(t)^{-q}.
\end{equation}
Hence, taking $q > 1$,
\begin{equation}
\begin{split}
& \sum_{t = 6}^n \frac{1}{t} \sum_{s = 1}^\infty t^{-1/s} (1-\alpha)^s\\
 & \leq \left( \left( \frac{1}{e} \frac{p+q}{1-p} \right)^{ \frac{p+q}{1-p} } + \frac{1}{\alpha} \left( \frac{1}{e \alpha} \frac{q}{p} \right)^{ \frac{q}{p} } \right) \sum_{t = 3}^n \frac{1}{t} \ln(t)^{-q} \\
& \leq \left( \left( \frac{1}{e} \frac{p+q}{1-p} \right)^{ \frac{p+q}{1-p} } + \frac{1}{\alpha} \left( \frac{1}{e \alpha} \frac{q}{p} \right)^{ \frac{q}{p} } \right) \int_{e}^n \frac{1}{t} \ln(t)^{-q} dt \\
& = \left( \left( \frac{1}{e} \frac{p+q}{1-p} \right)^{ \frac{p+q}{1-p} } + \frac{1}{\alpha} \left( \frac{1}{e \alpha} \frac{q}{p} \right)^{ \frac{q}{p} } \right) \frac{1 - \ln(n)^{1-q}}{q-1} \\
& \leq \left( \left( \frac{1}{e} \frac{p+q}{1-p} \right)^{ \frac{p+q}{1-p} } + \frac{1}{\alpha} \left( \frac{1}{e \alpha} \frac{q}{p} \right)^{ \frac{q}{p} } \right) \frac{1}{q-1}.
\end{split}
\end{equation}
At this point, taking $q = 2p$ and $p = 0.55$ yields
\begin{equation}
\sum_{t = 3}^n \frac{1}{t} \sum_{s = 1}^\infty t^{-1/s} (1-\alpha)^s \leq 29.9628 + \frac{ 5.41341 }{ \alpha^3 },
\end{equation}
which, rounding up, completes the result.
\end{proof}

\begin{proposition}\label{prop:log-bound}
For $Q > 0$, and $0 \leq \epsilon < 1$, the following bound holds:
\begin{equation}
\frac{1}{\ln(1 + Q(1 - \epsilon))} \leq \frac{1}{\ln(1 + Q)} + \frac{\epsilon}{1-\epsilon} \frac{Q}{ (1+Q) \ln(1+Q)^2}.
\end{equation}
\end{proposition}

\begin{proof}[Proof of Proposition \ref{prop:log-bound}]
Let $A(Q, \epsilon)$ denote the RHS of the above, $B(Q, \epsilon)$ denote the left. We adopt the physicists' convention of denoting the partial derivative of $F$ with respect to $x$ as $F_x$.

Note, $A(Q, 0) \leq B(Q,0)$. Hence, it suffices to demonstrate that $A_\epsilon \leq B_\epsilon$ over this range or, since they are both positive,
\begin{equation}
\frac{A_\epsilon}{B_\epsilon} = \frac{(1+Q)(1-\epsilon)^2 \ln(1 + Q)^2}{(1 + Q(1-\epsilon))\ln(1 + Q(1-\epsilon))^2} \leq 1.
\end{equation}
We take, for convenience, $\delta = 1 - \epsilon$, and want to show that for $0 \leq \delta \leq 1$:
\begin{equation}
\frac{(1+Q)\delta^2 \ln(1 + Q)^2}{(1 + Q\delta)\ln(1 + Q\delta)^2} \leq 1.
\end{equation}
The above inequality holds when $\delta = 1$. Taking $C(\delta, Q)$ as the above simplified ratio, it suffices to show that $C_\delta \geq 0$. Simplifying this inequality and canceling the positive factors, it is equivalent to show that $-2 Q \delta + (2 + Q \delta) \ln(1 + Q \delta ) \geq 0$, or taking $x = Q \delta > 0$,
\begin{equation}
\ln(1 + x) \geq \frac{2x}{2 + x}.
\end{equation}
This is a fairly standard and easily verified inequality for $\ln$. This completes the proof.
\end{proof}

\fontsize{9.5pt}{10.5pt} \selectfont 
\bibliographystyle{apsr} 
\bibliography{mab2015}

\end{document}